\documentclass[11pt]{article}
\pdfoutput=1
\usepackage{graphicx} 
\usepackage{subfigure}
\usepackage[square,sort,comma,numbers]{natbib}

\usepackage{hyperref}
\usepackage{fullpage}
\usepackage{url}

\usepackage{epstopdf}

\usepackage{amsmath}
\usepackage{amssymb}
\usepackage{amsthm}
\usepackage{amstext,amssymb,amsmath}
\usepackage{amsthm}
\usepackage{bm}
\usepackage{bbm}
\usepackage{todonotes}
\usepackage[noend]{algorithmic}
\usepackage{algorithm}

\usepackage{color,soul}

\newcommand{\tensor}{\otimes}

\newcommand{\Ex}{\mathbf{E}}
\newtheorem{lem}{Lemma}[section]

\newtheorem{thm}[lem]{Theorem}
\newtheorem{cor}[lem]{Corollary}

\newtheorem{claim}[lem]{Claim}

\newcommand{\colideset}[1]{\mathcal{E}(#1)}
\newcommand{\Hyp}{\mathcal{H}}
\newcommand{\one}{\mathbbm{1}}
\newcommand{\eqdef}{\stackrel{def}=}
\newcommand{\N}{\mathcal{N}}
\renewcommand{\Re}{\mathbb{R}}
\newcommand{\eps}{\varepsilon}
\newcommand{\doh}[2]{\frac{\partial #1}{\partial #2}}

\newcommand{\myOr}{{\bigvee}}

\newcommand{\myAnd}{{\bigwedge}}

\newcommand{\inputdim}{d}
\newcommand{\projdim}{d'}

\newcommand{\firstlayer}{p}
\newcommand{\margin}{\gamma}
\newcommand{\sparsity}{k}
\newcommand{\polynomialsparsity}{s}
\newcommand{\polynomialdegree}{p}

\newcommand{\cee}{c}
\newcommand{\x}{\mathbf x}
\newcommand{\y}{\mathbf y}
\newcommand{\z}{\mathbf z}
\newcommand{\w}{\mathbf w}
\newcommand{\e}{\mathbf{e}}

\newcommand{\numhashes}{t}
\newcommand{\rangesize}{m}
\newcommand{\relu}{\operatorname*{Relu}}
\newcommand{\supp}{{\cal S}}
\newcommand{\sepsparsity}{s}

\def\abovestrut#1{\rule[0in]{0in}{#1}\ignorespaces}
\def\belowstrut#1{\rule[-#1]{0in}{#1}\ignorespaces}
\def\abovespace{\abovestrut{0.20in}}

\def\belowspace{\belowstrut{0.10in}}

\title{Sketching and Neural Networks}

\author{
\vspace{1cm}
  Amit Daniely\thanks{Email: amitdaniely@google.com} \and
  Nevena Lazic\thanks{Email: nevena@google.com} \and
  Yoram Singer\thanks{Email: singer@google.com} \and
  Kunal Talwar\thanks{Email: kunal@google.com. Author for correspondences.}
}

\begin{document}

\maketitle
\thispagestyle{empty}

\begin{abstract}
High-dimensional sparse data present computational and statistical
challenges for supervised learning.  We propose compact linear sketches for
reducing the dimensionality of the input, followed by a single layer neural
network. We show that any sparse polynomial function can be computed, on
nearly all sparse binary vectors, by a single layer neural network that takes a
compact sketch of the vector as input. Consequently, when a set of sparse
binary vectors is approximately separable using a sparse polynomial, there
exists a single-layer neural network that takes a short
sketch as input and correctly classifies nearly all the points. 
Previous work has proposed using sketches to reduce dimensionality while preserving
 the hypothesis class. However, the sketch size has an
exponential dependence on the degree in the case of polynomial classifiers.
In stark contrast, our approach of using improper learning, using a larger hypothesis class
 allows the sketch size to have a logarithmic
dependence on the degree. Even in the linear case, our approach allows us to
improve on the pesky $O({1}/{{\margin}^2})$ dependence of random
projections.  We empirically show that our approach leads to more compact
neural networks than related methods such as feature hashing at equal
or better performance.
\end{abstract}

\setcounter{page}{1}

\newpage

\section{Introduction}

In many supervised learning problems, input data are high-dimensional and
sparse. The high dimensionality may be inherent in the task, for example a
large vocabulary in a language model, or the result of creating hybrid
conjunction features. Applying standard supervised learning techniques to such
datasets poses statistical and computational challenges, as high dimensional
inputs lead to models with a very large number of parameters.  For example, a
linear classifier for $\inputdim$-dimensional inputs has  $\inputdim$
weights and a linear multiclass predictor for $\inputdim$-dimensional vectors has 
has $\inputdim$ weights per class. In the case of a neural network with
$\firstlayer$ nodes in the first hidden layer, we get $\inputdim\firstlayer$
parameters from this layer alone.  Such large models can often lead to slow
training and inference, and may require larger datasets to ensure low generalization error.

One way to reduce model size is to project the data into a lower-dimensional
space prior to learning. Some of the proposed methods for reducing dimensionality are
random projections, hashing, and principal component analysis (PCA).
These methods typically attempt to project the data in a way that preserves
the hypothesis class.  While effective, there are inherent
limitations to these approaches. For example, if the data consists of linearly
separable unit vectors, Arriaga and Vempala~\cite{ArriagaV2006} show that
projecting data into $O(1/\margin^2)$ dimensions suffices to
preserve linear separability, if the original data had a margin $\gamma$. 
However, this may be too large for a small
margin $\margin$. Our work is motivated by the question: 
{\it could fewer dimensions suffice?} Unfortunately, it can be shown that
$\Omega(1/\margin^2)$ dimensions are needed in order to preserve
linear separability, even if one can use arbitrary embeddings (see Section~\ref{sec:lbproper}). 
It would appear therefore that the answer is a
resounding no.

In this work, we show that using improper learning, allowing for a slightly larger hypothesis class
allows us to get a positive answer. In the simplest case, we show that for
linearly separable $\sparsity$-sparse inputs, one can create a  $O(\sparsity
\log\frac \inputdim \delta)$-dimensional sketch of each input and guarantee
that a neural network with a single hidden layer taking the sketches as input, can correctly classify
$1\!-\!\delta$ fraction of the inputs. We show
that a simple non-linear function can ``decode'' every binary feature of the
input from the sketch. Moreover, this function can be implemented by applying
a simple non-linearity to a linear function. In the case of $0$-$1$ inputs, one can use a rectified linear unit (ReLU) that
is commonly used in neural networks. Our results also extend to sparse polynomial
functions.

In addition, we show that this can be achieved with sketching matrices that
are sparse. Thus the sketch is very efficient to compute and does not
increase the number of non-zero values in the input by much.  This result is in
contrast to the usual dense Gaussian projections. In fact, our sketches are
simple linear projections. Using a non-linear ``decoding'' operation allows us to improve on previous work.

We present empirical evidence that our approach leads to more compact  neural networks than existing methods such as feature hashing and Gaussian random projections. We leave open the question of how such sketches affect the difficulty of learning such networks.

In summary, our contributions are:
\begin{itemize}
  \item We show that on sparse binary data, any linear function is computable on most inputs by a small single-layer neural network on a compact linear sketch of the data.
  \item We show that the same result holds also for sparse polynomial functions with only a logarithmic dependence on the degree of the polynomial, as opposed to the exponential dependence of previous work. To our knowledge, this is the first technique that provably works with such compact sketches.
  \item We empirically demonstrate that on synthetic and real datasets, our approach leads to smaller models, and in many cases, better accuracy. The practical message stemming from our work is that a compact sketch built from multiple hash-sketches, in combination with a neural network, is sufficient to learn a small accurate model.
  \end{itemize}

\section{Related Work}

\paragraph{Random Projections.} Random Gaussian projections are by now a
classical tool in dimensionality reduction.  For general vectors, the
Johnson-Lindenstrauss Lemma~\cite{Johnson1984} implies that random
Gaussian projection into $O(\log(1/\delta)/\eps^2)$ dimensions preserves the inner product between a pair of unit
vectors up to an additive factor $\eps$, with probability $1\!-\!\delta$. A long line of work has sought
sparser projection matrices with similar guarantees; see ~\cite{Achlioptas2003,
AilonC2009, Matousek2008, DasguptaKS2010, BravermanOR2010, KaneN2014,
ClarksonW2013}.

\paragraph{Sketching.} Research in streaming and sketching algorithms has
addressed related questions. Alon et al.~\cite{AlonMS1999} showed a simple hashing
based algorithm to get unbiased estimators for the Euclidean norm in the
streaming setting. Charikar et al.~\cite{CharikarCF2004} showed an algorithm for the
heavy-hitters problem based on the {\em Count Sketch}. Most relevant to our
works is the {\em Count-min sketch} of Cormode and Muthukrishnan~\cite{CormodeM2005a, CormodeM2005b}.

\paragraph{Projections in Learning.}
Random projections have been used in machine learning at least since the work of Arriaga
and Vempala~\cite{ArriagaV2006}. For fast estimation of a certain class of kernel
functions, sampling has been proposed as a dimensionality reduction
technique in~\cite{Kontorovich2007} and~\cite{RahimiR2007}.
Shi et al.~\cite{Shietal2009} propose using a count-min sketch to reduce
dimensionality while approximately preserving inner products for sparse
vectors. Weinberger et al.~\cite{Weinbergeretal2009} use the count-sketch to get an
unbiased estimator for the inner product of sparse vectors and prove strong
concentration bounds. Previously, Ganchev and Dredze~\cite{GanchevD2008} had shown
empirically that hashing is effective in reducing model size without
significantly impacting performance. Hashing had also been used in Vowpal
Wabbit~\cite{LangfordLS2007}. Talukdar and Cohen~\cite{TalukdarC2014} also use the
count-min sketch in graph-based semi-supervised learning. Pham and Pagh~\cite{PhamP2013} showed that a count sketch of a tensor power of a
vector could be quickly computed without explicitly computing the tensor
power, and applied it to fast sketching for polynomial kernels.

\paragraph{Compressive Sensing.}
Our works is also related to the field of Compressive Sensing. For
$\sparsity$-sparse vectors, results in this area, see
e.g.~\cite{Donoho2006,CandesTa06}, imply that a $\sparsity$-sparse vector
$\x \in\Re^{\inputdim}$ can be reconstructed, w.h.p. from a projection of
dimension $O(\sparsity\ln \frac \inputdim \sparsity)$. However, to our
knowledge, all known decoding algorithms with these parameters involve
sequential adaptive decisions, and are not implementable by a low depth
neural network. Recent work by Mousavi et al.~\cite{MousaviPB2015}
empirically explores using a deep network for decoding in compressive
sensing and also considers learnt non-linear encodings to adapt to the
distribution of inputs.

\paragraph{Parameter Reduction in Deep Learning.}
Our work can be viewed as a method for reducing the number of parameters in
neural networks. Neural networks have become ubiquitous in many machine
learning applications, including speech recognition, computer vision, and
language processing tasks(see ~\cite{Hinton2012,Krizhevsky2012,Sermanet2013,Vinyals2014} for a few notable examples).  These successes
have in part been enabled by recent advances in scaling up deep networks,
leading to models with millions of parameters~\cite{Dean2012,Krizhevsky2012}. However, a drawback of such large models is
that they are very slow to train, and difficult to deploy on mobile and
embedded devices with memory and power constraints.
Denil et al.~\cite{Denil2013} demonstrated significant redundancies in the
parameterization of several deep learning architectures. They reduce the
number of parameters by training low-rank decompositions of weight matrices.
Cheng et al.~\cite{Cheng2015} impose circulant matrix structure on fully connected
layers. Ba and Caruana~\cite{Ba2014} train shallow networks to predict the log-outputs
of a large deep network, and Hinton et al.~\cite{HintonVD2015} train a small network
to match smoothed predictions of a complex deep network or an ensemble of
such models. Collines and Kohli~\cite{Collins2014} encourage zero-weight connections using
sparsity-inducing priors, while others such as~\cite{Lecun1989,Hassibi1993,Han2015}
use techniques for pruning weights.  HashedNets~\cite{Chen2015} enforce
parameter sharing between random groups of network parameters.  In contrast
to these methods, sketching only involves applying a sparse, linear
projection to the inputs, and does not require a specialized learning
procedure or network architecture.

\section{Notation} For a vector $\x \in \Re^\inputdim$, the support of $\x$
is denoted $\supp(\x) = \{i : x_i \neq 0\}$. The $p$-norm of $\x$ is denoted
$\|\x\|_p = \left(\sum_i \|x_i\|^p\right)^{\frac 1 p}$ and $\|\x\|_0 =
|\supp(\x)|$. We denote by $$ head_k(\x) =
\operatorname*{arg\,min}_{\displaystyle \y : \|\y\|_0 \leq \sparsity}
\|\x-\y\|_1 $$ the closest vector to $\x$ whose support size is $\sparsity$
and the residue $tail_k(\x) = \x - head_k(\x)$.
Let $B_{\inputdim,\sparsity}$ represent the set of all $\sparsity$-sparse
binary vectors $$B_{\inputdim,\sparsity} = \{ \x \in \{0,1\}^{\inputdim} :
\|\x\|_{0} \leq \sparsity\}.$$  Let $\Re^{+}_{\inputdim,\sparsity,\cee}$
represent the set $$\Re^{+}_{\inputdim,\sparsity,\cee} = \{\x \in
\Re_{+}^{\inputdim} : \|tail_\sparsity(\x)\|_1 \leq \cee\} \, ,$$ and let
$\Re_{\inputdim,\sparsity,\cee}$ represent the set
$$\Re_{\inputdim,\sparsity,\cee} = \{\x \in \Re^{\inputdim} :
\|tail_\sparsity(\x)\|_1 \leq \cee\} \, .$$
We examine datasets where feature vectors are sparse and come from
$B_{\inputdim,\sparsity}$. Alternatively, the feature can be ``near-sparse''
and come from $\Re^{+}_{\inputdim,\sparsity,\cee}$ or
$\Re_{\inputdim,\sparsity,\cee}$, for some parameters
$\sparsity,\inputdim,\cee$, with $\sparsity$ typically being much smaller
than $\inputdim$.
Let $\Hyp_{\inputdim,\sepsparsity}$ denote the set
$$\Hyp_{\inputdim,\sepsparsity} = \{\w \in \Re^{\inputdim} : \|\w\|_0 \leq
 \sepsparsity\} \,.$$
We denote by $\N_{n_{1}}(f)$ the family of feed-forward neural networks with
one hidden layer containing $n_{1}$ nodes with $f$ as the non-linear
function applied at each hidden unit, and a linear function at the output
layer.  Similarly, $\N_{n_{1},n_{2}}(f)$ designates neural networks with two
hidden layers with $n_{1}$ and $n_{2}$ nodes respectively.

\section{Sketching} \label{sec:sketching}
We will use the following family of randomized sketching algorithms based on the count-min sketch. Given a
parameter $\rangesize$, and a hash function $h : [\inputdim] \rightarrow
[\rangesize]$, the sketch $Sk_{h}(\x)$ is a vector $\y$ where
\begin{align*}
  y_l &= \sum_{i : h(i) = l} x_i\,.
\end{align*}
We will use several such sub-sketches to get our eventual sketch. It will be
notationally convenient to represent the sketch as a matrix with the
sub-sketches as columns. Given an ordered set of hash functions
$h_{1},\ldots, h_{\numhashes} \eqdef h_{1:\numhashes}\,$, the sketch
$Sk_{h_{1:\numhashes}}(\x)$ is defined as a matrix $Y$, where the
$j$'th column $\y_j = Sk_{h_j}(\x)$.

When $\x \in B_{\inputdim,\sparsity}\,$, we will use a boolean version where the
sum is replaced by an {\sc OR}. Thus $BSk_{h}$ is a vector $\y$ with
$$y_l = \operatorname*{\myOr}_{\displaystyle i : h(i) = l} x_i \, . $$
The sketch $BSk_{h_{1:t}}(\x)$ is
defined analagously as a matrix $Y$ with $\y_j = BSk_{h_j}(\x)$.
Thus a sketch $BSk_{h_{1:t}}(\x)$ is a matrix $[\y_{1},\ldots,\y_{t}]$ where
the $j$'th column is $\y_{j} \in \Re^{\rangesize}$. We define the following
decoding procedures:
\begin{align*}
Dec(\y, i ; h) &\eqdef y_{h(i)}\\
DecMin(Y, i ; h_{1:t}) &\eqdef \min_{j \in [t]}(Dec(\y_{j}, i ; h_{j})) \, .
\end{align*}
In the Boolean case, we can similarly define,
\begin{align*}
DecAnd(Y, i ; h_{1:t}) &\eqdef
	\operatorname*{\myAnd}_{j\in [t]}(Dec(\y_{j}, i ; h_{j}))\,.
\end{align*}
When it is clear from context, we omit the hash functions
$h_j$ from the arguments in $Dec$, $DecMin$, and $DecAnd$.

The following theorem summarizes the important property of these sketches.
To remind the reader, a set of hash functions $h_{1:t}$ from $[d]$ to $[m]$
is {\em pairwise independent} if for all $i\neq j$ and $a,b\in[m]$,
$\Pr[h_i=a \wedge h_j=b] = \frac{1}{m^2}$. Such hash families can be easily
constructed (see e.g.~\cite{MitzenmacherU2005}), using $O(\log m)$ random
bits. Moreover,  each hash can be evaluated using $O(1)$ arithmetic operations over
 $O(\log d)$-sized words.
\begin{thm}
  \label{thm:bsk}
Let $\x \in B_{\inputdim,\sparsity}$ and for $j\in[t]$ let
$h_{j} : [\inputdim] \rightarrow [\rangesize]$ be drawn uniformly and independently
from a pairwise independent distribution with $\rangesize=e\sparsity$.
Then for any $i$,
\begin{align*}
\Pr[DecMin(Sk_{h_{1:t}}(\x), i) \neq x_{i}] \leq e^{-t},\\
\Pr[DecAnd(BSk_{h_{1:t}}(\x), i) \neq x_{i}] \leq e^{-t}\,.
\end{align*}
\end{thm}
\begin{proof}
Fix a vector $\x \in B_{\inputdim,\sparsity}\,$.
For a specific $i$ and $h$, the decoding
$Dec(\y, i ; h)$ is equal to $y_{h(i)}$. Let us denote the set of collision indices of $i$ as
$$\colideset{i} \eqdef \{i'\neq i: h(i') = h(i)\} \,.$$
Then, we can rewrite
$y_{h(i)}$ as $x_i + \sum_{i'\in\colideset{i}} x_{i'}$. By pairwise
independence of $h$,
\begin{align*}
	\Ex\left[\sum_{i'\in\colideset{i}} x_{i'}\right] &=
		\sum_{i': x_{i'} = 1} \Pr[h(i') = h(i)]
    \leq \frac{\sparsity}{\rangesize} = \frac{1}{e} \,,
\end{align*}
since the sum is over at most $\sparsity$ terms, and each term is $\frac{1}{\rangesize}$. Thus by Markov's inequality, it follows that for any $j \in [t]$,
\begin{align*}
	\Pr[Dec(Sk_{h_j}(\x), i ; h_j) \neq x_{i}] \leq \frac{1}{e}\,.
\end{align*}
Moreover, it is easy to see that
$DecMin(Sk_{h_{1:t}}(\x),i)$ equals $x_i$ unless for each $j \in t$,
$$Dec(Sk_{h_j}(\x), i ; h_j) \neq x_{i}\,.$$
Since the $h_j$'s are drawn independently, it follows that
\begin{align*}
	\Pr[DecMin(Sk_{h_{1:t}}(\x), i) \neq x_{i}] \leq e^{-t}\,.
\end{align*}
The argument for $BSk$ is analogous, with $+$ replaced by $\vee$ and $\min$
replaced by $\wedge$.  \end{proof}

\begin{cor}
\label{cor:bcor}
Let $\w \in \Hyp_{\inputdim,\sepsparsity}$ and $\x \in
B_{\inputdim,\sparsity}$. For $t = \log (\sepsparsity/\delta)$, and
$\rangesize=e\sparsity$, if $h_{1},\ldots,h_{t}$ are drawn uniformly and independently
from a pairwise independent distribution, then
\begin{align*}
\Pr\left[\sum_{i} w_{i} \, DecMin(Sk_{h_{1:t}}(\x), i) \neq
\w^\top\x\right] \leq \delta,\\
\Pr\left[\sum_{i} w_{i} \, DecAnd(BSk_{h_{1:t}}(\x), i) \neq
\w^\top\x\right] \leq \delta \,.
\end{align*}
\end{cor}

In Appendix~\ref{app:generalsketches}, we prove the following extensions of
these results to vectors $\x$ in $\Re^{+}_{\inputdim,\sparsity,\cee}$ or in
$\Re_{\inputdim,\sparsity,\cee}\,$. Here $DecMed(\cdot,\cdot)$ implements
the median of the individual decodings.
\begin{cor} \label{cor:poscor}
Let $\w \in \Hyp_{\inputdim,\sepsparsity}$ and
$\x \in \Re^{+}_{\inputdim,\sparsity,\cee}\,$.
For $t = \log (\sepsparsity/\delta)$, and
$\rangesize=e(\sparsity +\frac 1 \eps)$, if $h_{1},\ldots,h_{t}$ are
drawn uniformly and independently from a pairwise independent distribution,
then
\begin{align*}
\Pr\!\!\left[\left|\sum_{i} w_{i} DecMin(Sk_{h_{1:t}}(\x),i) -
	\w^{\top}\x\right| \geq \eps c \|\w\|_{1}\right] \!\leq\!\delta
\end{align*}
\end{cor}
\begin{cor}
\label{cor:gencor}
Let $\w \in \Hyp_{\inputdim,\sepsparsity}$ and
$\x \in \Re_{\inputdim,\sparsity,\cee},$.
For $t = \log (\sepsparsity/\delta)$, and $\rangesize=4e^2(\sparsity+\frac 1 \eps)$, if $h_{1},\ldots,h_{t}$ are drawn uniformly and independently from a pairwise independent distribution, then
\begin{align*}
\Pr\!\!\left[\left|\sum_{i} w_{i} \, DecMed(Sk_{h_{1:t}}(\x),i) -
	\w^{\top}\x\right| \geq \eps \cee \|\w\|_{1}\right] \!\leq\! \delta
\end{align*}
\end{cor}

\section{Sparse Linear Functions} \label{sec:sparse_linear}

\begin{figure*}[t!]
\label{fig:sketch}
\begin{center}
\includegraphics[width=5in]{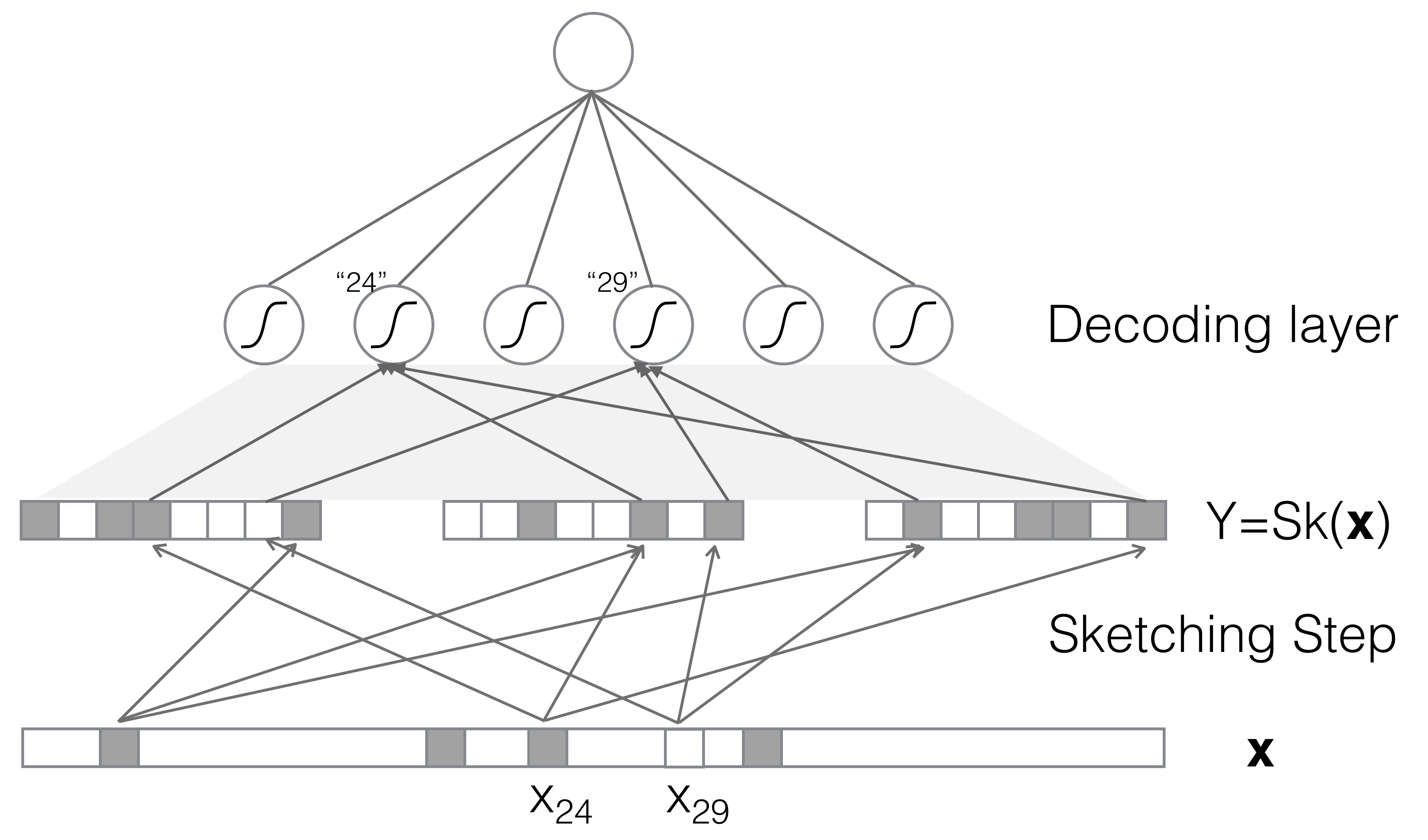}
\caption{Schematic description of neural-network sketching. The sparse
vector $\x$ is sketched to a sketch using $t=3$ hash functions, with $m=8$.
The shaded squares correspond to $1$'s. This sketching step is random and
not learned. The sketch is then used as an input to the single-layer neural
network that is trained to predict $\w^\top \x$ for a sparse weight vector
$\w$. Note that each node in the hidden layer of a candidate network
corresponds to a coordinate in the support of $\w$. We have labelled two
nodes ``24'' and ``29'' corresponding to decoding for $x_{24}$ and $x_{29}$
and shown the non-zero weight edges coming into them.}
\end{center}
\end{figure*}

Let $\w \in \Hyp_{\inputdim,\sepsparsity}$ and $\x \in
B_{\inputdim,\sparsity}$. As before we denote the sketch matrix by $Y
= BSk_{h_{1},\ldots,h_{t}}(\x)$ for $\rangesize,t$ satisfying the conditions of
Corollary~\ref{cor:bcor}. We will argue that there exists a network  in $\N_{\sepsparsity}(\relu)$
that takes $Y$ as input and outputs $\w^{\top}\x$ with high probability (over the randomness of the sketching process). Each hidden unit in our network corresponds to an index of a non-zero weight
$i \in \supp(\w)$ and implements $DecAnd(Y, i)$. The top-layer weight for the hidden unit corresponding to $i$ is simply $w_{i}$.

It remains to show that the first layer of the neural network can implement $DecAnd(Y,i)$. Indeed, implementing the
\texttt{And} of $t$ bits can be done using nearly any non-linearity. For
example, we can use a $\relu(a)=\max\{0,a\}$ with a fixed bias of $t-1$:
for any set $T$ and binary matrix $Y$ we have,
$$\displaystyle \operatorname*{\myAnd}_{(l,j)\in T} Y_{lj} =
	\relu\left(\displaystyle \sum_{l,j}
		\one\left[(l,j) \in T\right] \, Y_{lj} -
	\left(|T|-1\right)\right) \,.$$

Using Corollary~\ref{cor:bcor}, we have the following theorem:
\begin{thm}
\label{thm:boolnet}
For every $\w \in \Hyp_{\inputdim,\sepsparsity}$  there exists a set of weights
for a network $N \in \N_{\sepsparsity}(\relu)$ such that for each $\x \in
B_{\inputdim,\sparsity}$,
\begin{align*}
	Pr_{h_{1:{t}}} [N(BSk_{h_{1:t}}(\x)) = \w^\top \x] \geq 1-\delta \,,
\end{align*}
as long as $\rangesize = e\sparsity$ and $t=\log(\sepsparsity/\delta)$.
Moreover, the weights coming into each node in the hidden layer are in
$\{0,1\}$ with at most $t$ non-zeros.
\end{thm}
The final property implies that when using $\w$ as a linear classifier, we
get small generalization error as long as the
number of examples is at least
$\Omega(\sepsparsity(1+t\log \rangesize t))$.. This can be proved, e.g., using
  standard compression arguments: each such model can be represented
  using only $\sepsparsity t \log (\rangesize t)$ bits in addition to
  the representation size of $\w$. Similar bounds hold when we use
  $\ell_1$ bounds on the weight coming into each unit. Note that even for $s=d$ (i.e. $\w$ is unrestricted), we get non-trivial input compression.

For comparison, we prove the following result for Gaussian projections in the appendix~\ref{app:gaussian}. In this case, the model weights in our construction are not sparse.
\begin{thm}
\label{thm:boolgaussnet}
For every $\w \in \Hyp_{\inputdim,\sepsparsity}$  there exists a set of weights
for a network $N \in \N_{\sepsparsity}(\relu)$ such that for each $\x \in
B_{\inputdim,\sparsity}$,
\begin{align*}
	Pr_{h_{1:{t}}} [N(G\x)) = \w^\top \x] \geq 1-\delta \,,
\end{align*}
as long as $G$ is a random $\rangesize \times \inputdim$ Gaussian matrix, with $\rangesize \geq 4\sparsity\log(\sepsparsity/\delta)$.
\end{thm}

To implement $DecMin(\cdot)$, we need to use a slightly
non-conventional non-linearity. For a weight vector $\z$ and input
vector $\x$, a $\min$ gate implements $\min_{i: z_i \neq 0} z_i
x_i\,$. Then, using Corollary~\ref{cor:poscor}, we get the following theorem.
\begin{thm}
For every $\w \in \Hyp_{\inputdim,\sepsparsity}$ there exists a set of
weights for a network $N \in \N_{\sepsparsity}(\min)$ such that for each
$\x \in \Re^{+}_{\inputdim,\sparsity,\cee}$ the following holds,
\begin{align*}
Pr_{h_{1:t}}\left[\,|N(Sk_{h_{1},\ldots,h_{t}}(\x)) - \w^{\top}\x|
	\geq \eps \cee \|\w\|_{1}\,\right] \leq \delta \,.
\end{align*}
as long as $\rangesize = e(k+\frac{1}{\eps})$ and $t=\log(\sepsparsity/\delta)$.
Moreover, the weights coming into each node in the hidden layer are binary
with $t$ non-zeros.
\end{thm}
\noindent
For real vectors $\x$, the non-linearity needs to implement a median.
Nonetheless, an analogous result still holds.

\section{Representing Polynomials}

In the boolean case, Theorem~\ref{thm:boolnet} extends immediately to
polynomials. Suppose that the decoding $DecAnd(BSk_{h_{1:t}}(\x), i)$
gives us $x_{i}$ and similarly $DecAnd(BSk_{h_{1:t}}(\x), j)$ gives us
$x_{j}$. Then $x_i \wedge x_j$ is equal to the $\operatorname*{And}$ of the
two decodings. Since each decoding itself is an $\operatorname*{And}$ of $t$
locations in $BSk_{h_{1:t}}(\x)$, the over all decoding for $x_i \wedge x_j$
is an $\operatorname*{And}$ of at most $2t$ locations in the sketch. More
generally, for any set $A$ of indices, the conjunction of the variables in the
set $A$ satisfies
$$\displaystyle \operatorname*{\myAnd}_{i \in A} x_i =
	\operatorname*{\myAnd}_{(l,j) \in T_A} \!\!\!\! Y_{lj}
		\; \mbox{ where } \; T_A = \bigcup_{i \in A} T_i \,.
$$ Since an $\operatorname*{And}$ can be
implemented by a ReLU, the following result holds.
\begin{thm}
\label{thm:conjboolnet}
Given $\w \in \Re^{\sepsparsity}$, and sets
$A_1,\ldots,A_{\sepsparsity}\subseteq [d]$, let $g :
\{0,1\}^{\inputdim} \rightarrow \Re$ denote the polynomial function
\begin{align*}
g(\x) = \sum_{j=1}^\sepsparsity w_j \prod_{i \in A_j} x_i =
	\sum_{j=1}^\sepsparsity w_j \operatorname*{\myAnd}_{i \in A_j} x_i\,.
\end{align*}
Then there exists a set of weights
for a network $N \in \N_{\sepsparsity}(\relu)$ such that for each $\x \in
B_{\inputdim,\sparsity}$,
\begin{align*}
	Pr_{h_{1:{t}}} [N(Sk_{h_{1:t}}(\x)) = g(\x)] \geq 1-\delta \,,
\end{align*}
as long as $\rangesize = e\sparsity$ and $t=\log(|\cup_{j \in [\sepsparsity]} A_j|/\delta)$.
Moreover, the weights coming into each node in the hidden layer are in
$\{0,1\}$ with at most $t\cdot\left(\sum_{j \in [\sepsparsity]} |A_j|\right)$ non-zeroes
overall. In particular, when $g$ is a degree-$\polynomialdegree$
polynomial, we can set $t=\log(p\sepsparsity/\delta)$, and each hidden
unit has at most $\polynomialdegree t$ non-zero weights.
\end{thm}

This is a setting where we get a significant advantage over
proper learning. To our knowldege, there is no analog of this result for Gaussian projections. Classical sketching approaches would use a sketch of $\x^{\tensor \polynomialdegree}$, which is a $\sparsity^{\polynomialsparsity}$-sparse vector over binary vectors of
dimension $\inputdim^{\polynomialsparsity}$.  Known sketching techniques such as~\cite{PhamP2013}
would construct a sketch of size $\Omega(\sparsity^{\polynomialsparsity})$. Practical techniques such as Vowpal Wabbit also construct cross features by explicitly building them and have this exponential dependence. In stark contrast,  neural networks allow us to get away with a logarithmic dependence on $p$.

\section{Deterministic Sketching}

A natural question that arises is whether the parameters above can improved. We show that if we allow large scalars in the sketches, one can construct a deterministic $(2k+1)$-dimensional sketch from which a shallow network can reconstruct any monomial. We will also show a lower bound of $k$ on the required dimensionality. 

For every $\x\in B_{d,k}$ define a degree $2k$ univariate real
polynomial by,
\begin{eqnarray*}
p_\x(z) &=& 1 - (k+1)\!\!\prod_{\{i\mid x_i=1\}}(z-i)^2.
\end{eqnarray*}
It is easy to verify that this construction satifies the following.
\begin{claim}
\label{cla:polyval}
Suppose that $\x \in B_{d,k}$, and let $p_\x(\cdot)$ be defined as above. If $x_j = 1$, then $p_\x(j)=1$. If $x_j=0$, then $p_\x(j)\le - k$.
\end{claim}
Let the coeffients of $p_\x(z)$ be $a_i(\x)$ so that
\begin{eqnarray*}
p_\x(z) \eqdef	\sum^{2k}_{i=0} a_{i}(\x)z^i \,.
\end{eqnarray*}
Define the deterministic sketch $DSk_{d,k}:B_{d,k}\to \Re^{2k+1}$ as follows,
\begin{equation}\label{eq:det_sk_pol_def}
DSk_{d,k}(\x) = (a_{\x,0},\ldots,a_{\x,2k})
\end{equation}
For a non-empty subset $A\subset [d]$ and $\y\in\Re^{2k+1}$ define
\begin{equation}
DecPoly_{d,k}(\y,A) = \frac{\displaystyle
	\sum_{j\in A} \sum_{i=0}^{2k} y_i j^i }{|A|} \,.
\end{equation}
\begin{thm}
For every $\x \in B_{d,k}$ and a non-empty set $A\subset [d]$ we have
\[
\prod_{j\in A}x_j = \relu(DecPoly_{d,k}(DSk_{d,k}(\x),A)) \,.
\]
\end{thm}
\begin{proof}
We have that
\begin{eqnarray*}
DecPoly_{d,k}(DSk_{d,k}(\x),A) &=&
	\frac{\displaystyle \sum_{j\in A}\sum_{i=0}^{2k} a_{\x,i}j^i }{|A|} \\
&=& \frac{\sum_{j\in A}p_\x(j)}{|A|} \,.
\end{eqnarray*}
In words, the decoding is the average value of $p_\x(j)$ over the indices $j \in A$. Now first suppose that $\prod_{j\in A}x_j=1$. Then for each $j\in A$ we have $x_j = 1$ so that by Claim~\ref{cla:polyval}, $p_\x(j)=1$. Thus the average $DecPoly_{d,k}(DSk_{d,k}(\x),A)=1$.

On the other hand, if $\prod_{j\in A}x_j=0$ then for some $j\in A$, say $j^*$, $x_{j^*} = 0$. In this case, Claim~\ref{cla:polyval} implies that $p_\x(j^*)\le -k$. For every other $j$, $p_\x(j) \leq 1$, and each $p_\x(j)$ is non-negative only when $x_j = 1$, which happens for at most $k$ indices $j$. Thus the sum over non-negative $p_\x(j)$ can be no larger than $k$. Adding $p_\x(j^*)$ gives us zero, and any additional $j$'s can only further reduce the sum. Thus the average in non-positive and hence the Relu is zero, as claimed.
\end{proof}
The last theorem shows that $B_{d,k}$ can be sketched in $\Re^{q}$, where
$q=2k+1$, such that arbitrary products of variables be decoded by applying
a linear function followed by a ReLU. It is natural to ask what is
the smallest dimension $q$ for which such a sketch exists. The following
theorem shows that $q$ must be at least $k$. In fact, this is true even if we
only require to decode single variables.
\begin{thm}
Let $Sk:B_{d,k}\to\Re^q$ be a mapping such that for every $i\in [d]$ there is
$\w_i\in\Re^q$ satisfying $x_i = \relu(\langle\w_i, Sk(\x)\rangle)$ for each
$\x\in B_{d,k}$, then $q$ is at least $k$.
\end{thm}
\begin{proof}
Denote $X=\{\w_1,\ldots,\w_d\}$ and let $\mathcal H \subset \{0,1\}^X$ be the
function class consisting of all functions of the form
$h_{\x}(\w_i)=\mathrm{sign} (\langle\w_i, Sk(\x)\rangle)$ for $\x\in B_{d,k}$.
On one hand, $\mathcal H$ is a sub-class of the class of linear separators
over $X\subset\Re^q$, hence $\mathrm{VC}(\mathcal H) \le q$. On the other
hand, we claim that $\mathrm{VC}(\mathcal H)\ge k$ which establishes the
proof. In order to prove the claim it suffices to show that
the set $A=\{\w_1,\ldots, \w_k\}$ is shattered. Let
$B\subset A$ let $\x$ be the indicator vector of $B$. We claim that the
restriction of $h_\x$ to $A$ is the indicator function of $B$. Indeed,
we have that,
\begin{eqnarray*}
h_{\x}(\w_i) &=& \mathrm{sign} (\langle\w_i, Sk(\x)\rangle)
\\
&=& \mathrm{sign} (\relu (\langle\w_i, Sk(\x)\rangle))
\\
&=& x_i
\\
&=& \one[i\in B]
\end{eqnarray*}
\end{proof}
We would like to note that both the endcoding and the decoding of
determinitic sketched can be computed efficiently, and the dimension
of the sketch is smaller than the dimension of a random sketch.
We get the following corollaries.
\begin{cor}
\label{cor:boolnet_det}
For every $\w \in \Hyp_{\inputdim,\sepsparsity}$  there exists a set of weights
for a network $N \in \N_{\sepsparsity}(\relu)$ such that for each $\x \in
B_{\inputdim,\sparsity}$, 
$	N(DSk_{d,k}(\x)) = \w^\top \x$.
\end{cor}
\begin{cor}
Given $\w \in \Re^{\sepsparsity}$, and sets
$A_1,\ldots,A_{\sepsparsity}\subseteq [d]$, let $g :
\{0,1\}^{\inputdim} \rightarrow \Re$ denote the polynomial function
\begin{align*}
g(\x) = \sum_{j=1}^\sepsparsity w_j \prod_{i \in A_j} x_i =
	\sum_{j=1}^\sepsparsity w_j \operatorname*{\myAnd}_{i \in A_j} x_i\,.
\end{align*}
For any such $g$, there exists a set of weights
for a network $N \in \N_{\sepsparsity}(\relu)$ such that for each $\x \in
B_{\inputdim,\sparsity}$, 
$	N(DSk_{h_{1:t}}(\x)) = g(\x)$.
\end{cor}

Known lower bounds for compressed sensing~\cite{BaIPW2010} imply that any linear sketch has size at least $\Omega(k \log \frac{d}{k})$ to allow stable recovery. We leave open the question of whether one can get the compactness and decoding properties of our (non-linear) sketch while ensuring stability.

\section{Lower Bound for Proper Learning}
\label{sec:lbproper}
We now show that if one does not expand the hypothesis class, then even in
the simplest of settings of linear classifiers over $1$-sparse vectors, the
required dimensionality of the projection is much larger than the dimension
needed for improper learning. The result is likely folklore and we present
a short proof for completeness using concrete constants in the theorem and
its proof below.

\begin{thm}
Suppose that there exists a distribution over maps
$\phi: B_{\inputdim, 1} \rightarrow \Re^q$ and
$\psi : B_{\inputdim, \sepsparsity} \rightarrow \Re^q$ such that for
any $\x \in B_{\inputdim, 1}, \w \in B_{\inputdim, \sepsparsity}$,
\begin{align*}
  \Pr\left[sgn\left(\w^\top \x - \frac{1}{2}\right) =
		sgn\left(\psi(\w)^\top \phi(\x)\right)\right] \geq \frac{9}{10}\,,
  \end{align*}
where the probability is taken over sampling $\phi$ and $\psi$ from the distribuion. Then $q$ is $\Omega(\sepsparsity)$.
  \end{thm}
\begin{proof}
If the error is zero, a lower bound on $q$ would follow from standard VC
dimension arguments. Concretely, the hypothesis class consisting of
$h_\w(\x) = \{\e_i^\top\x : w_i = 1\}$ for all $\w \in B_{\inputdim, \sepsparsity}$ shatters the set
$\{\e_1,\ldots,\e_s\}$. If $sgn(\psi(\w)^\top \phi(\e_i)) = sgn(\w^\top\x - \frac 1 2)$ for each
$\w$ and $\e_i$, then the points $\phi(\e_i), i \in [s]$ are shattered by
$h_{\psi(\w)}(\cdot)$ where $\w \in B_{\inputdim, \sepsparsity}$, which is a
subclass of linear separators in $\Re^q$. Since linear separators in
$\Re^{q}$ have VC dimension $q$, the largest shattered set is no larger, and
thus $q \geq s$.

To handle errors, we will use the Sauer-Shelah lemma and show that the set
$\{\phi(\e_i) : i \in [s]\}$ has many partitions. To do so, sample $\phi, \psi$
from the distribution promised above and consider the set of points
$A = \{\phi(\e_{1}), \phi(\e_{2}), \ldots, \phi(\e_{\sepsparsity})\}$. Let
$W = \{\w_{1}, \w_{2}, \ldots, \w_{\kappa}\}$ be a set of $\kappa$ vectors in
$B_{\inputdim, \sepsparsity}$ such that,
\begin{description}
	\item{(a)} $\supp(\w_{i}) \subseteq [s]$
	\item{(b)} {\em Distance property} holds: $|\supp(\w_{i}) \bigtriangleup
\supp(\w_{j})| \geq  \frac{s}{4}$ for $i \neq j$.
\end{description}
Such a collection of vectors, with $\kappa = 2^{cs}$ for a positive constant $c$, can be shown to exist by a probabilistic
argument or by standard constructions in coding theory. Let
$H = \{ h_{\psi(\w)} : \w \in W\}$ be the linear separators defined by
$\psi(\w_{i})$ for $i \in W$. For brevity, we denote $h_{\psi(\w_{j})}$ by
$h_{j}$. We will argue that $H$ induces many different subsets of $A$.

Let $A_{j} = \{\y \in A : h_j(\x) = 1\} =
\{\x \in A: sgn(\psi(\w_{j})^{\top} \x) = 1\}$.
Let $E_{j} \subseteq A$ be the positions where the embeddings
$\phi,\psi$ fail, that is,
\begin{align*}
E_{j} = \{\phi(\e_{i}) : i \in [s], sgn(\w_{j}^{\top} \e_{i} - \frac 1 2) \neq sgn(\psi(\w_{j})^{\top} \phi(\e_{i})\}.
\end{align*}
Thus $A_{j} = \supp(w_{j}) \bigtriangleup E_{j}$. By assumption, $\Ex[|E_{j}|] \leq \frac{s}{10}$ for each $j$, where the expectation is taken over the choice of $\phi, \psi$. Thus $\sum_{j}\Ex[|E_{j}|] \leq s\kappa/10$. Renumber the $h_{j}$'s in increasing order of $|E_{j}|$ so that $|E_{1}| \leq |E_{2}| \leq \ldots \leq |E_{\kappa}|$. Due to the $E_{j}$'s being non-empty, not all $A_{j}$'s are necessarily distinct. Call a $j \in [\kappa]$ {\em lost} if $A_{j} = A_{j'}$ for some $j' \leq j$.

By definition, $A_{j} = \supp(\w_{j}) \bigtriangleup E_{j}$. If $A_{j} =
A_{j'}$, then the distance property implies that $E_{j} \bigtriangleup
E_{j'} \geq \frac{s}{4}$. Since the $E_{j}$'s are increasing in size, it
follows that for any lost $j$, $|E_{j}| \geq \frac{s}{8}$. Thus in
expectation at most $4\kappa/5$ of the $j$'s are lost. It follows that there is a
choice of $\phi,\psi$ in the distribution for which $H$ induces $\kappa/5$
distinct subsets of $A$. Since the VC dimension
of $H$ is at most $q$, the Sauer-Shelah lemma says that
\begin{align*}
\sum_{t \leq q} {s \choose t} \geq \frac{\kappa}{5} = \frac{2^{cs}}{5}\,.
\end{align*}
This implies that $q \geq c's$ for some absolute constant $c'$.
\end{proof}

Note that for the setting of the above example, once we scale the $\w_{j}$'s
to be unit vectors, the margin is $\Theta(\frac{1}{\sqrt{s}})$.
Standard results then imply that projecting to $\frac{1}{\gamma^{2}} =
\Theta(s)$ dimension suffices, so the above bound is tight.

For this setting, Theorem~\ref{thm:boolnet} implies that a sketch of size
$O(\log (\sepsparsity/\delta))$ suffices to correctly classify
$1\!-\!\delta$ fraction of the examples if one allows improper learning as
we do.

\section{Neural Nets on Boolean Inputs}
In this short section show that for boolean inputs (irrespective of
sparsity), any polynomial with
$\polynomialsparsity$ monomials can be represented by a neural network with one hidden layer of
$\polynomialsparsity$
hidden units. Our result is a simple improvement of Barron's theorem~ \cite{Barron1993,Barron1994}, for the
special case of sparse polynomial functions on $0$-$1$ vectors.
In contrast, Barron's theorem,
which works for arbitrary inputs, would require a neural network of size
$\inputdim\cdot\polynomialsparsity\cdot\polynomialdegree^{O(\polynomialdegree)}$ to learn an $\polynomialsparsity$-sparse degree-$\polynomialdegree$ polynomial.
The proof of the improvement is elementary and provided for completeness.
\begin{thm}
Let $\x \in \{0,1\}^\inputdim$, and  let $g :
\{0,1\}^{\inputdim} \rightarrow \Re$ denote the polynomial function
\begin{align*}
g(\x) = \sum_{j=1}^\sepsparsity w_j \prod_{i \in A_j} x_i =
	\sum_{j=1}^\sepsparsity w_j \operatorname*{\myAnd}_{i \in A_j} x_i\,.
\end{align*}
Then there exists a set of weights
for a network $N \in \N_{\sepsparsity}(\relu)$ such that for each $\x \in
\{0,1\}^\inputdim$, $N(\x) = g(\x)$.
Moreover, the weights coming into each node in the hidden layer are in
$\{0,1\}$.
\end{thm}
\begin{proof}
The $j$th hidden unit implements $h_j = \prod_{i \in A_j} x_i $. As before,
for boolean inputs, one can compute $h_{j}$ as $\relu(\sum_{i \in A_j} x_i -
|A_j| + 1)$. The output node computes $\sum_j w_j h_j$ where $h_j$ is the
output of $j$th hidden unit.
\end{proof}

\section{Experiments with synthetic data}
In this section, we evaluate sketches on synthetically generated datasets for
the task of polynomial regression. In all the experiments here, we assume
input dimension $d=10^4$, input sparsity $k=50$, hypothesis support $s=300$,
and $n=2\times10^5$ examples. We assume that only a subset of features
$\mathcal{I} \subseteq [d]$ are relevant for the regression task, with
$|\mathcal{I}|=50$. To generate an hypothesis, we select $s$ subsets of
relevant features $A_1,\ldots,A_{\sepsparsity}\subset \mathcal{I}$ each of
cardinality at most 3, and generate the corresponding weight vector $\w$ by
drawing corresponding $s$ non-zero entries from the standard Gaussian
distribution. We generate binary feature vectors $\x \in B_{d,k}$ as a mixture
of relevant and other features. Concretely, for each example we draw $12$
feature indices uniformly at random from $\mathcal{I}$, and the remaining
indices from $[d]$. We generate target outputs as $g(\x) + z$, where $g(\x)$
is in the form of the polynomial given in Theorem~\ref{thm:conjboolnet}, and
$z$ is additive Gaussian noise with standard deviation $0.05$. In all
experiments, we train on 90\% of the examples and evaluate the average
squared error on the rest.

We first examine the effect of the sketching parameters $m$ and $t$ on the
regression error. We generated sparse linear regression data using the above
described settings, with all feature subsets in $\mathcal{A}$ having
cardinality 1. We sketched the inputs using several values of $m$ (hash size)
and $t$ (number of hash functions). We then trained neural networks in
$\mathcal{N}_s(Relu)$ on the regression task. The results are shown in
Figure~\ref{fig:linreg_tm}. As expected, increasing the number of hash
functions $t$ leads to better performance. Using hash functions of size $m$
less than the input sparsity $k$ leads to poor results, while increasing hash
size beyond $m=ek$ (in this case, $m=ek \approxeq 136$) for reasonable $t$
yields only modest improvements.
\begin{figure}[t]
\centerline{
\includegraphics[trim = 0cm 0.9cm 0cm 1cm, clip,
	width=0.7\linewidth]{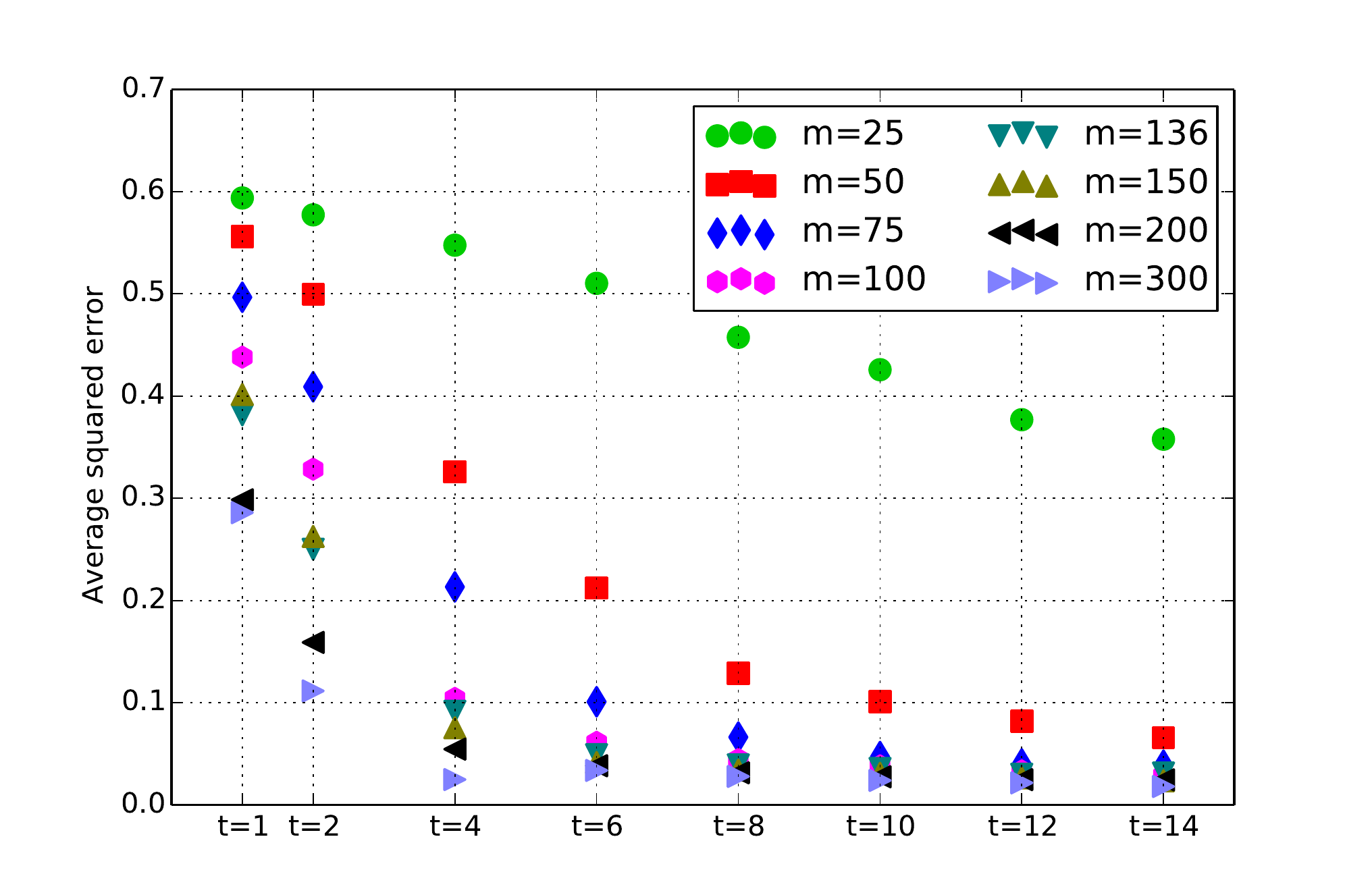}}
\vspace{-0.2cm}
\caption{The effects of varying $t$ and $m$ for
a single layer neural-network trained on sparse
linear regression data with sketched inputs.
}
\label{fig:linreg_tm}
\end{figure}

We next examine the effect of the decoding layer on regression performance. We
generated 10 sparse linear regression datasets and 10 sparse polynomial
regression datasets. The feature subsets in $\mathcal{A}$ had cardinality of 2
and 3. We then trained linear models and one-layer neural networks in
$\mathcal{N}_s(Relu)$ on original features and sketched features with $m=200$
for several values for $t$. The results are shown in
Figure~\ref{fig:decoding}. In the case of linear data, the neural network
yields notably better performance than a linear model. This suggests that
linear classifiers are not well-preserved after projections, as the
$\Omega(1/\gamma^2)$ projection size required for linear separability can be
large. Applying a neural network to sketched data allows us to use smaller
projections.

In the case of polynomial regression, neural networks applied to sketches
succeed in learning a small model and achieve significantly lower error than a
network applied to the original features for $t \geq 6$. This suggests that
reducing the input size, and consequently the number of model parameters, can
lead to better generalization. The linear model is a bad fit, showing that
our function $g(\x)$ is not well-approximated by a linear function. Previous
work on hashing and projections would imply using significantly larger
sketches for this setting.
\begin{figure}[t]
\centerline{
\includegraphics[trim = 0cm 0.7cm 0cm 0.8cm, clip,
	width=0.7\linewidth]{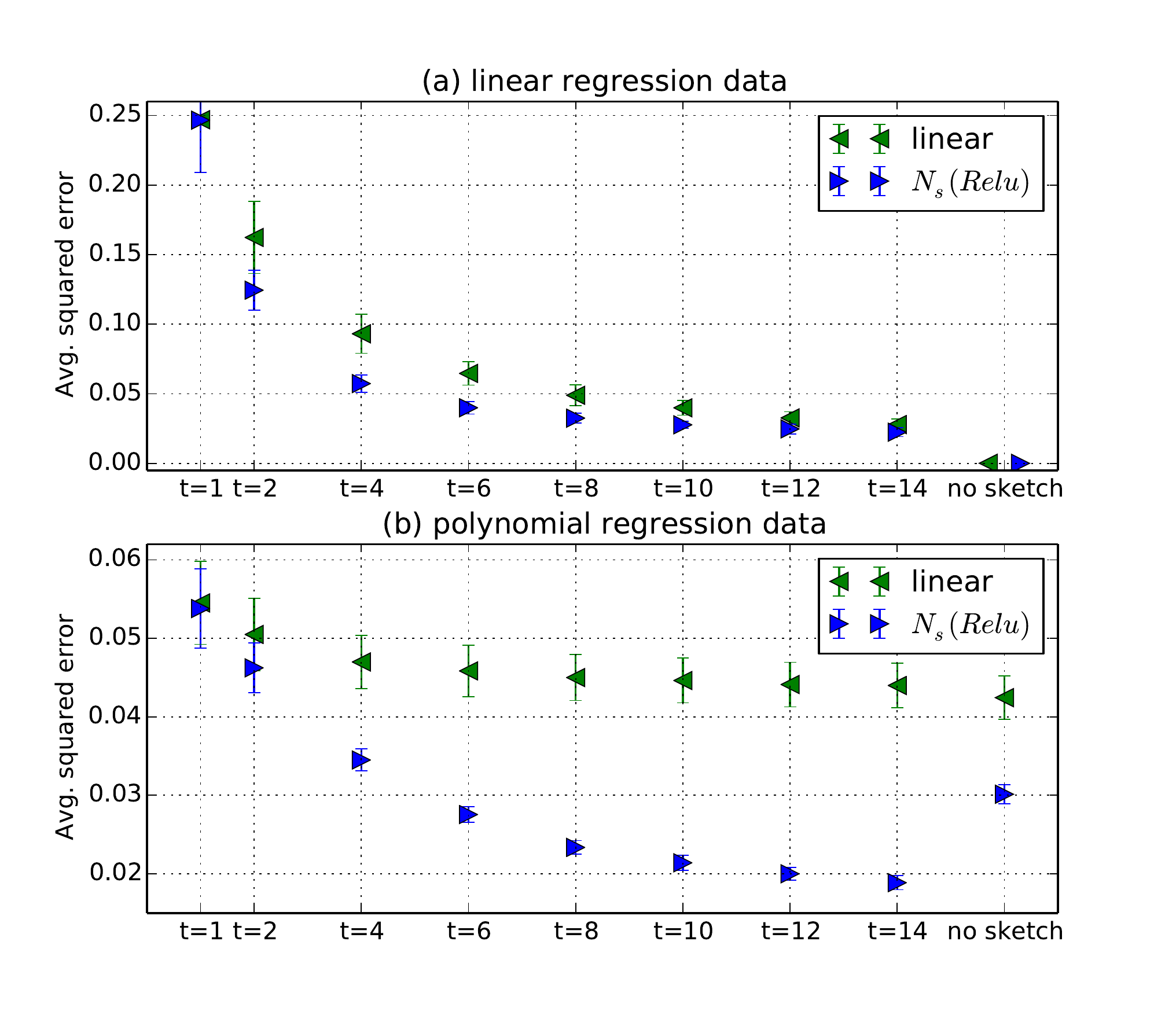}}
\vspace{-0.4cm}
\caption{Effect of the decoding layer on performance on linear and polynomial
regression data.}
\label{fig:decoding}
\end{figure}

To conclude this section, we compared sketches to Gaussian random projections.
We generated sparse linear and polynomial regression datasets with the same
settings as before, and reduce the dimensionality of the inputs to $1000$,
$2000$ and $3000$ using Gaussian random projections and sketches with $t \in
\{1, 2, 6\}$. We report the squared error averaged across examples and five
datasets of one-layer neural networks in Table~\ref{table:gaussproj}. The
results demonstrate that sketches with $t>1$ yield lower error than Gaussian
projections. Note also that Gaussian projections are dense and hence much
slower to train.
\begin{table}[th]
\begin{center}
\begin{small}
\begin{sc}
\begin{tabular}{lcccc}
\hline
\abovespace\belowspace
  & 1K & 2K & 3K\\
\hline
\abovespace
Gaussian    &  0.089   & 0.057 & 0.029   \\
Sketch $t=1$  & 0.087 & 0.049 & 0.031\\
Sketch $t=2$  & 0.072 & 0.041 & 0.023 \\
Sketch $t=6$  & 0.041 & 0.033 & 0.022 \\
\hline
Gaussian    &  0.043   & 0.037 & 0.034  \\
Sketch $t=1$  & 0.041 & 0.036 & 0.033\\
Sketch $t=2$  & 0.036 & 0.027 & 0.024 \\
Sketch $t=6$  & 0.032 & 0.022 & 0.018 \\
\hline
\end{tabular}
\end{sc}
\end{small}
\end{center}
\vskip -0.1in
\caption{Comparison of sketches and Gaussian random projections on the sparse linear regression task (top) and sparse polynomial regression task (bottom). See text for details.}
\label{table:gaussproj}
\end{table}

\begin{figure}[t]
\centerline{\includegraphics[width=0.7\linewidth]{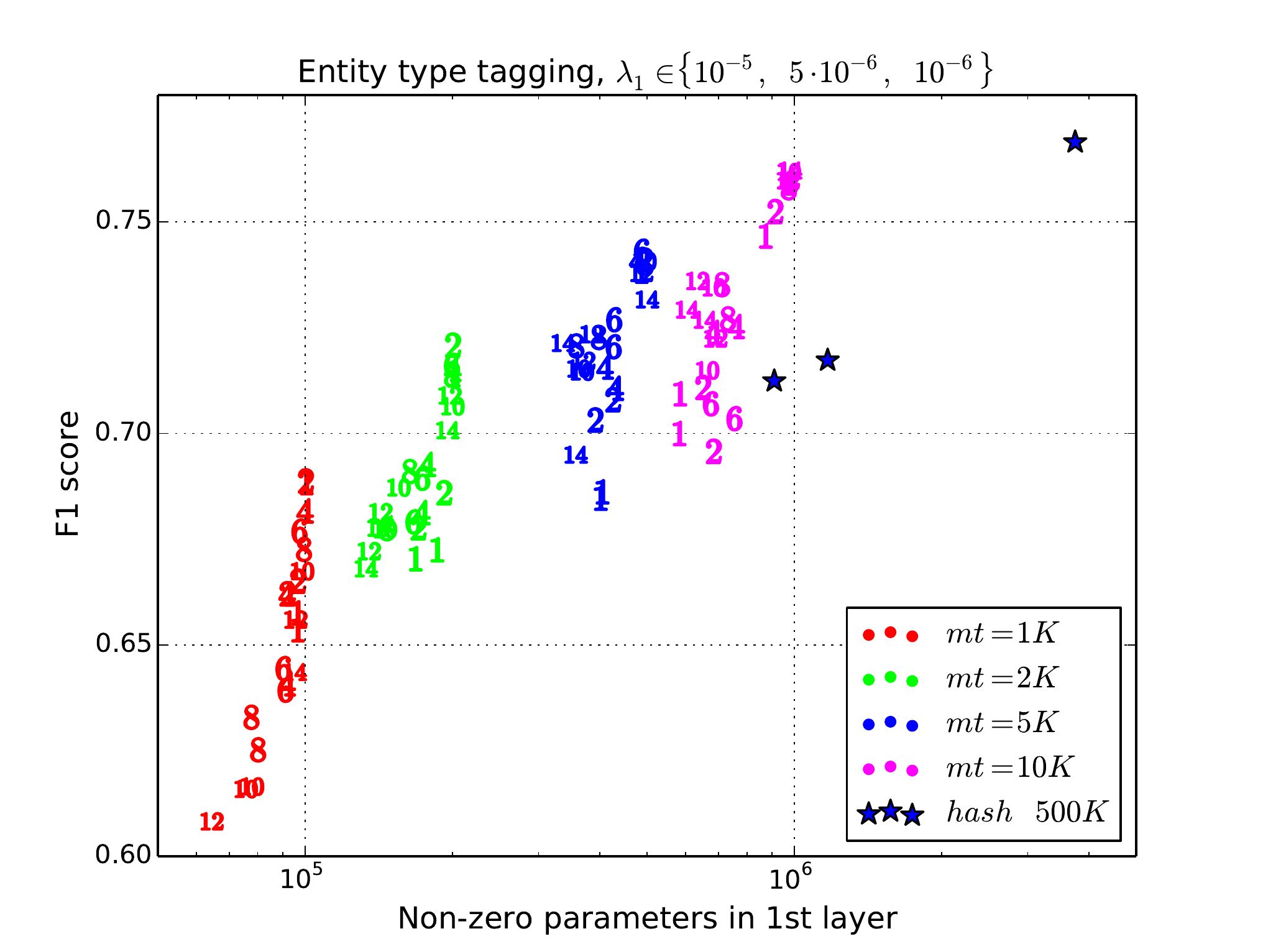}}
\vspace{-0.3cm}
\caption{F1 score vs. number of non-zero parameters in the first layer for
entity type tagging. Each color corresponds to a different sketch size and
markers indicate the number of subsketches $t$.}
\label{fig:sketch_finetype}
\end{figure}
\begin{figure}[th]
\centerline{\includegraphics[width=0.7\linewidth]{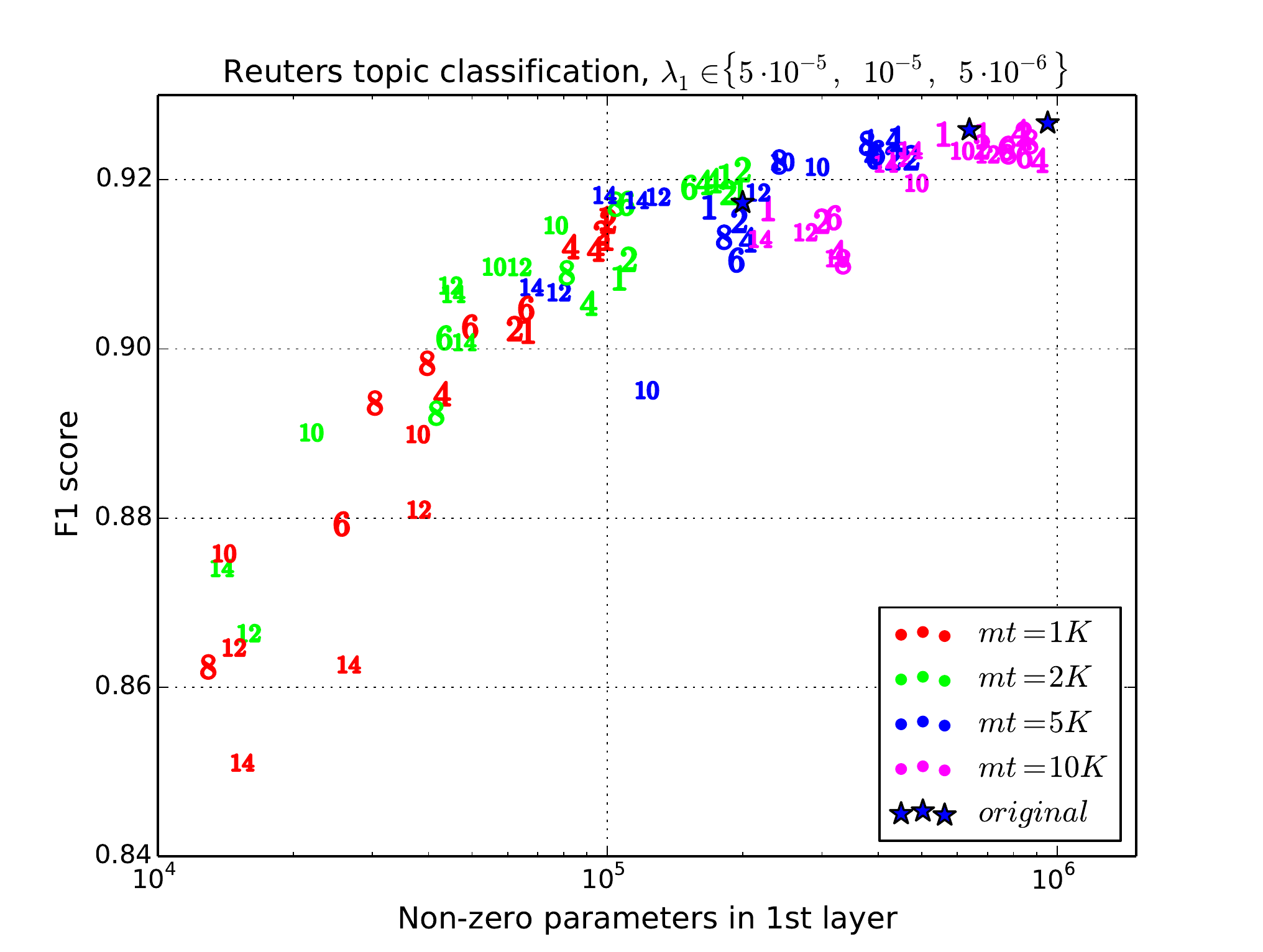}}
\vspace{-0.3cm}
\caption{F1 score vs. number of non-zero parameters in the first layer for Reuters topic classification. Each color corresponds to a different sketch size $tm$, and markers indicate the number of subsketches $t$.}
\label{fig:sketch_reuters}
\end{figure}
\begin{figure}[th]
\centerline{\includegraphics[width=0.7\linewidth]{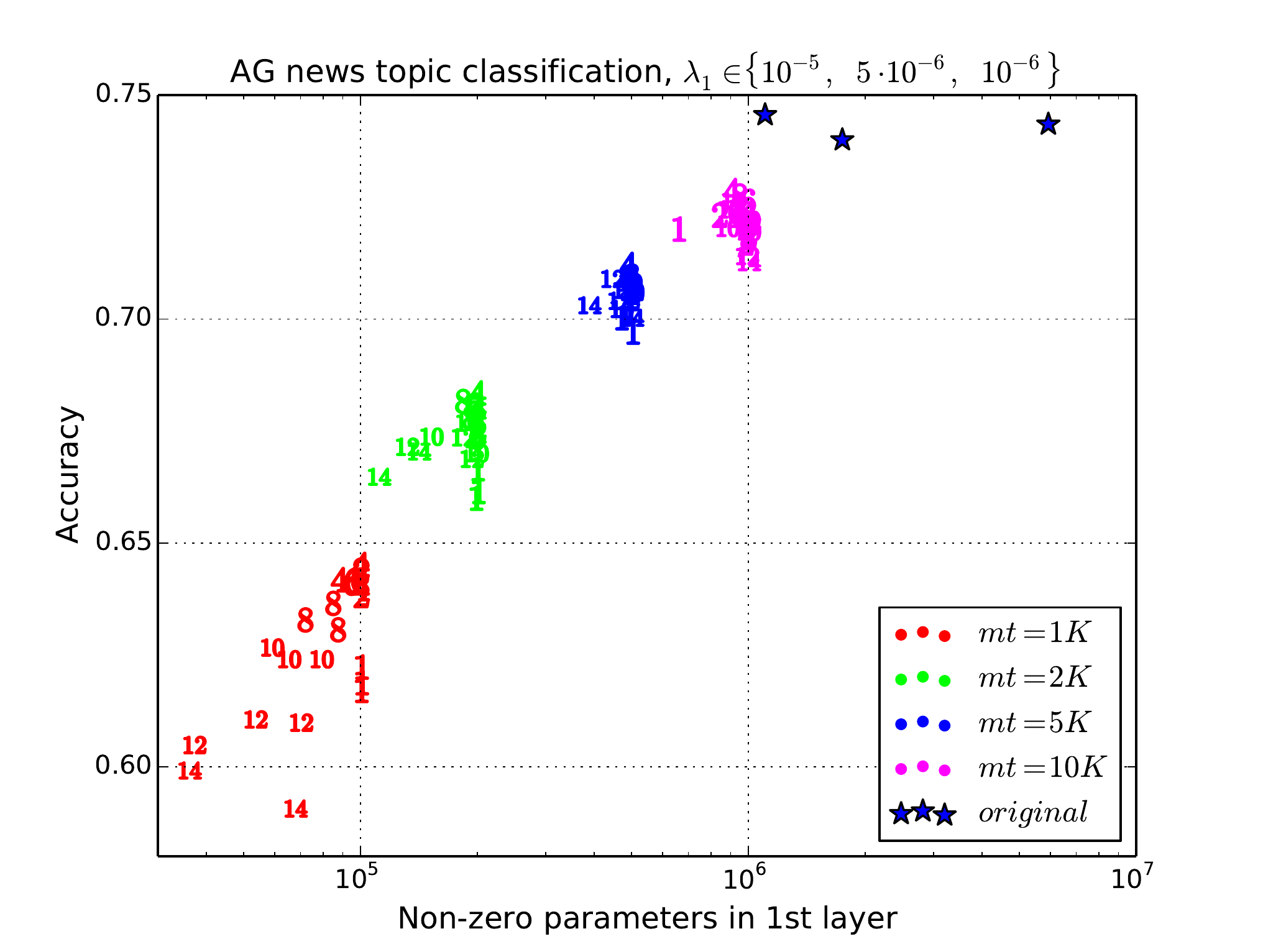}}
\vspace{-0.3cm}
\caption{AG news topic classification accuracy vs. number of non-zero parameters in the first layer. Each color corresponds to a different sketch size $tm$, and markers indicate the number of subsketches $t$.}
\label{fig:sketch_agnews}
\end{figure}
\section{Experiments with language processing tasks}
Linear and low degree sparse polynomials are often used for classification.
Our results imply that if we have linear or a sparse polynomial with
classification accuracy $1\!-\!\eps$ over some set of examples in
$B_{\inputdim,\sparsity} \times \{0,1\}$, then neural networks constructed to
compute the linear or polynomial function attain accuracy of at least
$1\!-\!\eps\!-\!\delta$ over the same examples. Moreover, the number of
parameters in the new network is relatively small by enforcing sparsity or
$\ell_{1}$ bounds for the weights into the hidden layers.  We thus get
generalization bounds with negligible degradation with respect to non-sketched
predictor. In this section, we evaluate sketches on the language processing
classification tasks described below.

\paragraph{Entity type tagging.}
Entity type tagging is the task of assigning one or more labels (such as
\emph{person}, \emph{location}, \emph{organization}, \emph{event}) to mentions
of entities in text. We perform type tagging on a corpus of new documents
containing 110K mentions annotated with 88 labels (on average, 1.7 labels per
mention). Features for each mention include surrounding words, syntactic and
lexical patterns, leading to a very large dictionary. Similarly to previous
work, we map each string feature to a 32 bit integer, and then further reduce
dimensionality using hashing or sketches. See~\cite{Gillick2014} for more
details on features and labels for this task.

\paragraph{Reuters-news topic classification.}
The Reuters RCV1 data set consists of a collection of approximately 800,000
text articles, each of which is assigned multiple labels. There are 4
high-level categories: Economics, Commerce, Medical, and Government (ECAT,
CCAT, MCAT, GCAT), and multiple more specific categories. We focus on training
binary classifiers for each of the four major categories. The input features
we use are binary unigram features. Post word-stemming, we get
data of approximately 113,000 dimensions. The feature vectors are very sparse,
however, and most examples have fewer than 120 non-zero features.

\paragraph{AG news topic classification.}
We perform topic classification on $680K$ articles from
\href{http://www.di.unipi.it/~gulli/AG_corpus_of_news_articles.html}{AG news
corpus}, labeled with one of 8 news categories: \emph{Business},
\emph{Entertainment}, \emph{Health}, \emph{Sci/Tech}, \emph{Sports},
\emph{Europe}, \emph{U.S.}, \emph{World}.  For each document, we extract
binary word indicator features from the title and description; in total, there
are ~210K unique features, and on average, 23 non-zero features per document.

\paragraph{Experimental setup.}
In all experiments, we use two-layer feed-forward networks with ReLU
activations and 100 hidden units in each layer. We use a softmax output for
multiclass classification and multiple binary logistic outputs for multilabel
tasks. We experimented with input sizes of 1000, 2000, 5000, and 10,000 and
reduced the dimensionality of the original features using sketches with $t \in
\{1, 2, 4, 6, 8, 10, 12, 14\}$ blocks. In addition, we experimented with
networks trained on the original features. We encouraged parameter sparsity in
the first layer using $\ell_1$-norm regularization and learn parameters using
the proximal stochastic gradient method. As before, we trained on 90\% of the
examples and evaluated on the remaining 10\%. We report accuracy values for
multiclass classification, and F1 score for multilabel tasks, with true
positive, false positive, and false negative counts accumulated across all
labels.

\paragraph{Results.}
Since one motivation for our work is reducing the number of parameters in
neural network models, we plot the performance metrics versus the number of
non-zero parameters in the first layer of the network. The results are shown
in Figures \ref{fig:sketch_finetype}, \ref{fig:sketch_reuters}, and
\ref{fig:sketch_agnews} for different sketching configurations and settings of
the $\ell_1$-norm regularization parameters ($\lambda_1$). On the entity type
tagging task, we compared sketches to a single hash function of size 500,000
as the number of the original features is too large. In this case, sketching
allows us to both improve performance and reduce the number of parameters. On
the Reuters task, sketches achieve similar performance to the original
features with fewer parameters. On AG news, sketching results in more compact
models at a modest drop in accuracy. In almost all cases, multiple hash
functions yield higher accuracy than a single hash function for similar model
size.

\section{Conclusions}

We have presented a simple sketching algorithm for sparse boolean inputs,
which succeeds in significantly reducing the dimensionality of inputs. A
single-layer neural network on the sketch can provably model any sparse linear or
polynomial function of the original input. For $k$-sparse vectors in
$\{0,1\}^d$, our sketch of size $O(k\log s/\delta)$ allows computing any
$s$-sparse linear or polynomial function on a $1\!-\!\delta$ fraction of the
inputs. The hidden constants are small, and our sketch is sparsity
preserving. Previous work required sketches of size at least $\Omega(s)$ in
the linear case and size at least $k^p$
for preserving degree-$p$ polynomials.

Our results can be viewed as showing a compressed sensing scheme for $0$-$1$
vectors, where the decoding algorithm is a depth-1 neural network. Our
scheme requires $O(\sparsity \log \inputdim)$ measurements, and we leave
open the question of whether this can be improved to $O(\sparsity \log
\frac{\inputdim}{\sparsity})$ in a stable way.

We demonstrated empirically that our sketches work well for both linear and
polynomial regression, and that using a neural network does improve over a
direct linear regression. We show that on real datasets, our methods lead to
smaller models with similar or better accuracy for multiclass and multilabel
classification problems. In addition, the compact sketches lead to fewer
trainable parameters and faster training.

\section*{Acknowledgements} We would like to thank Amir Globerson for
numerous fruitful discussion and help with an early version of the
manuscript.

\bibliography{deepsketch}
\bibliographystyle{alpha}

\appendix

\section{General Sketches} \label{app:generalsketches}
The results of Section~\ref{sec:sketching} extend naturally to positive and
to general real-valued $\x$.
\begin{thm}
\label{thm:rsk}
Let $\x \in \Re^{+}_{\inputdim,\sparsity,\cee}$ and let $h_{1},\ldots,h_{t}$
be drawn uniformly and independently from a pairwise independent distribution, for
$\rangesize=e(\sparsity+\frac{1}{\eps})$. Then for any $i$,
\begin{align*}
\Pr\left[DecMin(Sk_{h_{1},\ldots,h_{t}}(\x), i) \not\in [x_{i}, x_{i} + \eps \cee] \right]\leq \exp(-t).
\end{align*}
\end{thm}
\begin{proof}
Fix a vector $\x \in \Re^{+}_{\inputdim,\sparsity,\cee}$.
To remind the reader, for a specific $i$ and $h$, $Dec(\y, i ; h) \eqdef y_{h(i)}$, and we defined the set of collision indices $\colideset{i} \eqdef \{i'\neq i: h(i') = h(i)\}$.
Therefore, $y_{h(i)}$ is equal to
$x_i + \sum_{i'\in\colideset{i}} x_{i'} \geq x_i$. We next show that
\begin{align} \label{eq:oneh}
	\Pr\left[Dec(Sk_{h}(\x), i ; h) > x_{i}+\eps \cee\right]
		\leq 1/e\,.
\end{align}
We can rewrite
\begin{align*}
Dec(Sk_{h}(\x), i ; h) - x_{i} \; = \;
\sum_{i' \in \supp(head_{\sparsity}(\x)) \cap \colideset{i}}
	\hspace{-26pt} x_{i'} \quad +
	\sum_{i' \in \supp(tail_{\sparsity}(\x)) \cap \colideset{i}}
	\hspace{-26pt} x_{i'} \quad .
\end{align*}
By definition of $tail_k(\cdot)$, the expectation of the second term is
$\cee/\rangesize$. Using Markov's inequality, we get that
\begin{align*}
\Pr[\sum_{i' \in \supp(tail_{\sparsity}(\x)) \cap \colideset{i}}
		\hspace{-26pt} x_{i'} \quad > \eps \cee] \leq \frac{1}{\eps \rangesize}\,.
\end{align*}
To bound the first term, note that
\begin{align*}
	\Pr\left[\sum_{i' \in \supp(head_{\sparsity}(\x)) \cap \colideset{i}}
		\hspace{-26pt} x_{i'} \neq 0 \quad\right]
	\leq
	\Pr\left[\sum_{i' \in \supp(head_{\sparsity}(\x)) \cap \colideset{i}}
		\hspace{-26pt} 1 \neq 0 \quad\right]
	\leq
		\Ex\left[\sum_{i' \in \supp(head_{\sparsity}(\x)) \cap \colideset{i}}
			\hspace{-26pt} 1 \;\quad\right]
  \leq \frac k \rangesize\,.
\end{align*}
Recall that $\rangesize = e(\sparsity+1/\eps)$, then, using the union
bound establishes (\ref{eq:oneh}). The rest of the proof is identical
to Theorem~\ref{thm:bsk}.
\end{proof}
\begin{cor}
\label{cor:poscor-app}
Let $\w \in \Hyp_{\inputdim,\sepsparsity}$ and $\x \in
\Re^{+}_{\inputdim,\sparsity,\cee}\,$. For $t = \log (\sepsparsity/\delta)$, and
$\rangesize=e(\sparsity +\frac 1 \eps)$, if $h_{1},\ldots,h_{t}$ are drawn
uniformly and independently from a pairwise independent distribution, then
\begin{align*}
\Pr\left[|\sum_{i} w_{i} \cdot DecMin(Sk_{h_{1:t}}(\x),i) -
	\w^{\top}\x| \geq \eps c \|\w\|_{1}\right] \leq \delta\,.
\end{align*}
\end{cor}
The proof in the case that $\x$ may not be non-negative is only slightly
more complicated. Let us define the following additional decoding procedure,
\begin{align*}
DecMed\left(Y , i ; h_{1:t}\right)
	&\eqdef \operatorname*{Median}_{j \in [t]} Dec(\y_{j}, i ; h_{j}) \,.
\end{align*}
\begin{thm}
Let $\x \in \Re_{\inputdim,\sparsity,\cee}\,$, and let $h_{1},\ldots,h_{t}$ be
drawn uniformly and independently from a pairwise independent distribution, for
$\rangesize=4e^2(\sparsity+2/\eps)$. Then for any $i$,
\begin{align*}
\Pr\left[DecMed(Sk_{h_{1:t}}(\x),i) \not\in [x_{i}-\eps \cee, x_{i} +
	\eps \cee]\right] \leq e^{-t}\,.
\end{align*}
\end{thm}
\begin{proof}
As before, fix a vector $\x \in \Re_{\inputdim,\sparsity,\cee}$ and a specific
$i$ and $h$. We once again write
\begin{align*}
 Dec(Sk_{h}(\x), i ; h) - x_{i} \; = \;
	 \sum_{i' \in \supp(head_{\sparsity}(\x)) \cap \colideset{i}}
			\hspace{-26pt} x_{i'}  \;\quad +
	\sum_{i' \in \supp(tail_{\sparsity}(\x)) \cap \colideset{i}}
			\hspace{-26pt} x_{i'}  \;\quad .
\end{align*}

By the same argument as the proof of Theorem~\ref{thm:rsk}, the first term is
nonzero with probability $k/\rangesize$. The second term has expectation
in $[-\cee/\sepsparsity, \cee/\sepsparsity]$. Once again by Markov's
inequality,
\begin{align*}
\Pr\left[
	\sum_{i' \in \supp(tail_{\sparsity}(\x)) \cap \colideset{i}}
	\hspace{-26pt} x_{i'} > \eps \cee\right] \leq \frac{1}{\eps \rangesize}
	\;\; \mbox{ and } \;\;
\Pr\left[ \sum_{i'  \in \supp(tail_{\sparsity}(\x)) \cap \colideset{i}}
	\hspace{-26pt} x_{i'} < -\eps \cee\right] \leq \frac{1}{\eps \rangesize}\,.
\end{align*}
Recalling that $\rangesize = 4e^2(\sparsity+2/\eps)$, a union bound
establishes that for any $j$,
\begin{align*}
	\Pr[|Dec(Sk_{h_j}(\x),h_j, i) - x_{i}| > \eps \cee] \leq \frac{1}{4e^2}.
\end{align*}
Let $X_j$ be indicator for the event that
$|Dec(Sk_{h_j}(\x),h_j, i) - x_{i}| >\eps \cee$. Thus $X_1,\ldots, X_t$ are
binomial random variables with
$Pr[X_j=1] \leq \frac 1 {4e^2}$. Then by Chernoff's bounds
\begin{align*}
&\Pr\left[DecMed(Sk_{h_{1:t}}(\x),i)
	\not\in [x_{i}-\eps \cee, x_{i} + \eps \cee]\right] \\
		&\leq Pr\left[\sum_j X_j > t/2\right]\\
&\leq  \exp\left(-\left(\frac{1}{2}\ln \frac {1/2}{{1}/{4e^2}}
	+ \frac{1}{2} \ln \frac{{1}/{2}}{1-{1}/{4e^2}}\right)t\right)\\
&\leq  \exp\left(-\left(\frac{1}{2}\ln 2e^2
	+ \frac{1}{2} \ln \frac {1}{2}\right)t\right)\\
&\leq \exp(-t)\,.
\end{align*}
\end{proof}
\begin{cor}
\label{cor:gencor-app}
Let $\w \in \Hyp_{\inputdim,\sepsparsity}$ and $\x \in
\Re_{\inputdim,\sparsity,\cee},$.  For $t = \log (\sepsparsity/\delta)$, and
$\rangesize=4e^2(\sparsity+\frac 1 \eps)$, if $h_{1},\ldots,h_{t}$ are drawn
uniformly and independently from a pairwise independent distribution, then
\begin{align*}
\Pr\left[\left|\sum_{i} w_{i} \, DecMed(Sk_{h_{1:t}}(\x),i) -
	\w^{\top}\x\right| \geq \eps \cee \|\w\|_{1}\right]
	\leq \delta\,.
\end{align*}
\end{cor}

\section{Gaussian Projection} \label{app:gaussian}
In this section we describe and analyze a simple decoding algorithm for
Gaussian Projections.
\begin{thm}
Let $\x \in \Re^{\inputdim}$, and let $G$ be a random Gaussian matrix in
$\Re^{\projdim\times \inputdim}$. Then for any $i$, there exists a linear
function $f_i$ such that
\begin{align*}
\Ex_{G}\left[(f_i(G\x)-x_i)^2\right] &\leq \frac{\|\x - x_i \e_i\|_2^2}{\projdim}
\end{align*}
\end{thm}
\begin{proof}
Recall that $G \in \Re^{\projdim\times \inputdim}$ is a random Gaussian matrix
where each entry is chosen i.i.d. from $N(0,1/\projdim)$. For any $i$,
conditioned on $G_{ji}=g_{ji}$, we have that the random variable
$Y_j|G_{ji}=g_{ji}$ is distributed according to the following distribution,
\begin{align*}
(Y_j|G_{ji}=g_{ji}) &\sim g_{ji} x_i + \sum_{i' \neq i} G_{ji'}\,x_{i'}\\
&\sim g_{ji}x_i + N(0,\|\x-x_i\e_i\|_2^2/\projdim)\,.
\end{align*}
Consider a linear estimator for $x_i$:
\begin{align*}
\hat{x}_i &\eqdef
	\frac{ \sum_j \alpha_{ji} (y_j / g_{ji}) }{ \sum_j \alpha_{ji}}\,,
\end{align*}
for some non-negative $\alpha_{ji}$'s. It is easy to verify that for any
vector of $\alpha$'s, the expectation of $\hat{x}_i$, when taken over the
random choices of $G_{ji'}$ for $i'\neq i$, is $x_i$. Moreover, the variance
of $\hat{x}_i$ is
$$
	\frac{\|\x-x_i\e_i\|_2^2}{d} \,
	\frac{\sum_j (\alpha_{ji}/g_{ji})^2}{(\sum_j \alpha_{ji})^2} \,.
$$
Minimizing the variance of $\hat{x}_i$ w.r.t $\alpha_{ji}$'s gives us
$\alpha_{ji} \propto g_{ji}^2$. Indeed, the partial derivatives are,
\begin{align*}
\doh{\left(\sum_j (\alpha_{ji}/g_{ji})^2 - \lambda \sum_{j}
	\alpha_{ji}\right)}{\alpha_{ji}} &=
		\frac{2\alpha_{ji}}{g_{ji}^2} - \lambda\,,
\end{align*}
which is zero at $\alpha_{ji} = \lambda g_{ji}^2/2$. This choice of
$\alpha_{ji}$'s translates to
\begin{align*}
\Ex[(\hat{x}_i - x_i)^2] & =
	\frac{\|\x-x_i\e_i\|_2^2}{\projdim} \, \frac{1}{\sum_j g_{ji}^2} \;,
\end{align*}
which in expectation, now taken over the choices of $g_{ji}$'s, is at most
$\|\x-x_i\e_i\|_2^2/\projdim$. Thus, the claim follows.
\end{proof}

For comparison, if $\x \in B_{\inputdim,\sparsity}$, then the expected error
in estimating $x_i$ is $\sqrt{\frac{\sparsity-1}{\projdim}}$, so that taking
$\projdim = (\sparsity-1)\log\frac 1 \delta/{\eps^{2}}$ suffices to get a
error $\eps$ estimate of any fixed bit with probablity $1\!-\!\delta$. Setting
$\eps = \frac 1 2$, we can recover $x_{i}$ with probability $1\!-\!\delta$ for
$\x\in B_{\inputdim,\sparsity}$ with $d=4(\sparsity-1)\log \frac 1 \delta$. This implies Theorem~\ref{thm:boolgaussnet}.
However, note that the decoding layer is now densely connected to the input
layer. Moreover, for a $\sparsity$-sparse vectors $\x$ that are is necessarily
binary, the error grows with the ${2}$-norm of the vector $\x$, and can be
arbitrarily larger than that for the sparse sketch. Note that $G\,x$ still
contains sufficient information to recover $\x$ with error depending only on
$\|\x-head_\sparsity(\x)\|_{2}$. To our knowledge, all the known decoders are
adaptive algorithms, and we leave open the question of whether bounds
depending on the ${2}$-norm of the residual $(\x-head_\sparsity(\x))$ are
achievable by neural networks of small depth and complexity.

\end{document}